\documentclass[11pt]{article}
\usepackage[scale=.675,centering]{geometry}
\usepackage{amssymb,amsmath}
\usepackage{latexsym}
\usepackage[round]{natbib}
\usepackage{amsthm}
\usepackage{wrapfig}
\usepackage{eulervm}

\date{October 16, 2007}

\title{Probabilistic coherence and proper scoring rules%
\thanks{\copyright\, 2007  by the authors. This paper may be reproduced, in its
entirety, for non-commercial purposes.
Research supported by NSF grants PHY-0652854
to Lieb, and PHY-0652356 to Seiringer.
R.S. also acknowledges partial support from an A.\ P.\ Sloan Fellowship.}}

\author{%
Joel Predd\\ Rand Corporation \and
Robert Seiringer\\ Princeton University\and
Elliott H.\ Lieb \\ Princeton University \and
Daniel Osherson \\ Princeton University \and
Vincent Poor\\ Princeton University \and
Sanjeev Kulkarni\\ Princeton University}

\newcommand{\PROB}{\mbox{$\mu$}}
\newcommand{\CONV}{\mbox{\textit{conv}}}
\newcommand{\MID}{\mbox{$\ :\ $}}

\def\BE{\mbox{$\cal E$}}
\def\PE{\mbox{$\textsf{P}_{\!\!s}$}}
\newcommand{\CF}[1]{\mbox{$C_{#1}$}}

\newcommand\vecp{{\mathbold p}}
\newcommand\vecq{{\mathbold q}}
\newcommand\vecf{{\mathbold f}}
\newcommand\vecg{{\mathbold g}}
\newcommand\vecx{{\mathbold x}}
\newcommand\vecy{{\mathbold y}}
\newcommand\vecv{{\mathbold v}}
\newcommand\vecpi{{\mathbold \pi}}

\theoremstyle{definition}
\newtheorem{prop}{Proposition}
\newtheorem{cor}{Corollary}
\newtheorem{thm}{Theorem}

\newtheorem{defn}{Definition}

\newtheorem{lem}{Lemma}

\begin{document}

\maketitle
\begin{abstract}
\noindent
We provide self-contained proof of a theorem relating
probabilistic coherence of forecasts to their non-domination by
rival forecasts with respect to any proper scoring rule. The
theorem appears to be new but is closely
related to results achieved by
other investigators.
\end{abstract}

\section{Introduction}

Scoring rules measure the quality of a probability-estimate for
a given event, with lower
scores signifying probabilities that are closer to the
event's status ($1$ if it occurs, $0$ otherwise).
The sum of the scores for estimates $\vecp$ of a vector \BE\ of events
is called the ``penalty'' for $\vecp$.
Consider two potential defects in $\vecp$.
\begin{itemize}
\item There may be rival estimates $\vecq$ for \BE\ whose penalty
is guaranteed to be lower than the one for $\vecp$, regardless of
which events come to pass.
\item The events
in \BE\ may be related by inclusion or partition, and $\vecp$ might
violate constraints imposed by the
probability calculus (for example, that the estimate for an
event not exceed the estimate for any event that includes it).
\end{itemize}

\noindent
Building on the work of earlier investigators (see below), we
show that for a broad class of scoring rules known as
``proper'' the two defects are equivalent.
An exact statement appears as Theorem \ref{thm1}. To reach
it, we first explain key concepts intuitively (the next
section) then formally (Section \ref{frame}). Proof of the
the theorem proceeds via three propositions of independent
interest (Section \ref{threep}). We conclude with
generalizations of our results and an open question.

\section{Intuitive account of concepts}\label{sec1}

Imagine that you attribute probabilities $.6$ and $.9$ to events $E$ and $F$,
respectively, where $E\subseteq F$. It subsequently turns out that $F$ comes to
pass but not $E$. How shall we assess the perspicacity of your two estimates,
which may jointly be called a \textbf{probabilistic forecast}?
%
%
According to one method (due to \citealp{Brier50}) truth and falsity are coded
by $1$ and $0$, and your estimate of the chance of $E$ is assigned a score of
$(0 - .6)^2$ since $E$ did not come true (so your estimate should ideally have
been zero). Your estimate for $F$ is likewise assigned $(1 - .9)^2$ since it
should have been one. The sum of these numbers serves as overall
\textbf{penalty}.

Let us calculate your expected penalty for $E$ (prior to discovering the facts). With $.6$
probability you expected a score of $(1 - .6)^2$, and with the
remaining probability you expected a score of $(0 - .6)^2$,
hence your overall expectation was $.6(1 - .6)^2 + .4(0 - .6)^2
= .24$. Now suppose that you attempted to improve (lower) this
expectation by insincerely announcing $.65$ as the chance of
$E$, even though your real estimate is $.6$. Then
your expected penalty would be
$.6(1 - .65)^2 + .4(0 - .65)^2 = .2425$, worse than before.
Differential calculus reveals the general fact:

\begin{quote}
Suppose your probability for an event $E$ is $p$, that your
announced probability is $x$, and that your penalty is assessed
according to the rule: $(1 - x)^2$ if $E$ comes out true;
$(0 - x)^2$ otherwise. Then your expected
penalty is uniquely minimized by choosing $x = p$.
\end{quote}

\noindent
Our scoring rule thus encourages sincerity since
your interest lies in announcing probabilities that conform to
your beliefs. Rules like this are called \textbf{proper}. (We add
a continuity condition in our formal treatment, below.)
For an example of an improper rule, substitute absolute deviation
for squared deviation in the original scheme. According to the new
rule, your expected penalty for $E$ is
$.6|1 - .6| + .4|0 - .6| = .48$ whereas it drops to
$.6|1 - .65| + .4|0 - .65| = .47$ if you fib as before.

\begin{wraptable}{r}{0.39\textwidth}\begin{small}
\centering
\fbox{\begin{tabular}{c|c|c}
& \multicolumn{2}{c}{Forecast}\\
\begin{tabular}{c}Logical\\ possibilities\\\end{tabular} & original & rival\\\hline
\begin{tabular}{lcl}
E & = & T\\
F & = & T\\
\end{tabular} & $.17$ & $.205$\\\hline

\begin{tabular}{lcl}
E & = & F\\
F & = & T\\
\end{tabular} & $.37$ & $1.105$\\\hline

\begin{tabular}{lcl}
E & = &F\\
F & = & F\\
\end{tabular} & $1.17$ & $1.205$\\
\end{tabular}}
\caption{Penalties for two forecasts
in alternative possible realities}\end{small}
\end{wraptable}
%
Consider next the rival forecast of $.95$ for $E$ and $.55$ for $F$.
Because $E \subseteq F$, this forecast is inconsistent with the
probability calculus (or \textbf{incoherent}). Table 1 shows that the original forecast \textbf{dominates}
the rival inasmuch as its penalty is lower however the facts play out.
This association of incoherence and domination is not an accident.
No matter what proper scoring rule is in force, any incoherent
forecast can be replaced by a coherent one whose penalty
is lower in every possible
circumstance; there is no such replacement for a coherent forecast.
This fact is formulated as Theorem \ref{thm1} in the next section. It
can be seen as partial vindication of probability as an expression
of chance.\footnote{The other classic vindication involves
sure-loss contracts; see \cite{Skyrms00}.}

These ideas have been discussed before, first by \cite{def74}
who began the investigation of dominated forecasts and
probabilistic consistency (called \textbf{coherence}). His work
relied on the \textbf{quadratic scoring rule}, introduced
above.\footnote{For analysis of de Finetti's work, see
\citealp{Joyce98}. Note that some authors use the term
\textbf{inadmissible} to qualify dominated forecasts.}
\cite{Lindley82} generalized de Finetti's theorem to a broad
class of scoring rules. Specifically, he proved that for every
sufficiently regular generalization $s$ of the quadratic score,
there is a transformation $T :\Re \rightarrow \Re$ such that a
forecast $\vecf$ is not dominated by any other forecast with
respect to $s$ if and only if the transformation of $\vecf$ by
$T$ is probabilistically coherent. The reliance on the
transformation $T$, however, clouds the interpretation of
Lindley's theorem.

Fresh insight into proper scoring rules comes from relating them to a
generalization of metric distance known as \textbf{Bregman divergence}
\citep{Breg67}. This relationship was studied by \cite{Savage71}, albeit
implicitly, and more recently by \cite{Ban06} and \cite{Gneiting07}. So far
as we know, their results have yet to be connected to the issue of dominance.

To pull together the threads of earlier discussions,
the present work offers a self-contained account of the relations
among (i) coherent forecasts, (ii) Bregman divergences, and (iii)
domination with respect to proper scoring rules.
Only elementary analysis is presupposed. We begin by formalizing
the concepts introduced above.\footnote{%
For application of scoring rules to the assessment of opinion,
see \cite{Gneiting07} along with
\citet[\S 2.7.2]{Bernardo94} and references cited there.}

\section{Framework and Main Result}\label{frame}

Let $\Omega$ be a nonempty \textbf{sample space}. Subsets of $\Omega$ are
called \textbf{events}. Let \BE\ be a vector $(E_1, \cdots, E_n)$ of $n \ge 1$
events over $\Omega$.
We assume
that $\Omega$ and \BE\ have been chosen and are now fixed for the remainder of
the discussion. We require \BE\ to have finite dimension $n$ but otherwise
our results hold for any choice of sample space and events. In particular,
$\Omega$ can be infinite.
We rely on the usual notation $[0,1], \ (0,1),  \ \{0,1\}$ to denote, respectively, the closed interval
$\{x:0\leq x \leq 1\}$, the open interval $\{x:0 <  x < 1\}$ and the two-point set containing
$0,1$.

\begin{defn}\label{nd1}
Any element of $[0,1]^n$ is called a \textbf{(probability) forecast (for \BE)}.
A forecast $\vecf$ is \textbf{coherent} just in case there is a probability measure
\PROB\ over $\Omega$ such that for all $i\le n$, $f_i = \PROB(E_i)$.
\end{defn}

\noindent
A forecast is thus a list of $n$ numbers drawn from the unit interval. They are
interpreted as claims about the chances of the corresponding
events in \BE. The first event in \BE\ is assigned the probability given by the
first number ($f_1$) in $\vecf$, and so forth.
A forecast is coherent if it is consistent with some probability
measure over $\Omega$.

This brings us to scoring rules. In what follows, the numbers $0$ and $1$ are
used to represent falsity and truth, respectively.

\begin{defn}\label{def1p}
A function $s:\{0,1\}\times [0,1] \to [0,\infty]$ is said to be a \textbf{proper scoring
rule} in case
\begin{enumerate}\setlength{\parskip}{.0cm}
\item\label{deflpb} $ps(1,x) + (1 - p)s(0,x)$ is uniquely minimized at $x = p$ for all $p\in[0,1]$.
\item\label{deflpa} $s$ is continuous, meaning that for $i\in\{0,1\}$, $\lim_{n\to \infty} s(i,x_n) = s(i,x)$ for any sequence $x_n \in [0,1]$ converging to $x$.
\end{enumerate}
\end{defn}

\noindent
For condition \ref{def1p}(\ref{deflpb}), think
of $p$ as the probability you have in mind, and $x$ as the one you announce.
Then $ps(1,x) + (1 - p)s(0,x)$ is your expected score. Fixing $p$ (your genuine
belief), the latter expression is a function of the announcement $x$. Proper
scoring rules encourage candor by minimizing the expected score exactly when
you announce $p$.\footnote{Some authors call such rules \textit{strictly proper}.}

The continuity condition is consistent with $s$ assuming the value $+\infty$.
This can only occur for the arguments $(0,1)$ or $(1,0)$, representing categorically mistaken
judgment. For if $s(0,p)=\infty$ for some $p\neq 1$, then $ps(1,x)+(1-p)s(0,x)$ can not
have a unique minimum at $x=p$; similarly, $s(1,p)<+\infty$ for $p\neq 0$.
A typical example of an unbounded proper scoring rule is $s(i,x)=-\ln|i-x|$ \citep{Good52}.
A comparison of alternative rules is offered in \cite{Selten98}.

For an event $E$, we let \CF{E} be the characteristic function of $E$; that
is, for all $\omega\in\Omega$,  $\CF{E}(\omega) = 1$ if $\omega\in E$ and $0$
otherwise. Intuitively, $\CF{E}(\omega)$ reports whether $E$ is true or false if Nature
chooses $\omega$.

\begin{defn}\label{def2}
Given proper scoring rule $s$,
the \textbf{penalty} \PE\ based on $s$ for forecast $\vecf$ and $\omega\in\Omega$ is given by:
\begin{equation}\label{penalty}
\PE(\omega,\vecf) = \sum_{i \le n}s(\CF{E_i}(\omega),f_i).
\end{equation}
\end{defn}

\noindent
Thus, \PE\ sums the scores (conceived as penalties) for all the events under
consideration.
Henceforth, the proper scoring rule $s$ is regarded as given and fixed.
The theorem below holds for any choice we make.

\begin{defn}\label{def3w}
Let a forecast $\vecf$ be given.
\begin{enumerate}\setlength{\parskip}{.0cm}
\item\label{def3wa}
$\vecf$ is \textbf{weakly dominated} by a forecast $\vecg$ in case
$\PE(\omega,\vecg) \le \PE(\omega,\vecf)$ for all $\omega\in\Omega$.
\item\label{def3wb}
$\vecf$ is \textbf{strongly dominated} by a forecast $\vecg$ in case
$\PE(\omega,\vecg) < \PE(\omega,\vecf)$ for all $\omega\in\Omega$.
\end{enumerate}
\end{defn}

\noindent
Strong domination by a rival, coherent forecast $\vecg$ is the price to be paid for an
incoherent forecast $\vecf$. Indeed, we shall prove:

\begin{thm}\label{thm1} Let a forecast $\vecf$ be given.
\begin{enumerate}\setlength{\parskip}{.0cm}
\item If $\vecf$ is coherent then it is not weakly dominated by any
forecast $\vecg\neq \vecf$.
\item If $\vecf$ is incoherent then it is strongly dominated by some {\it coherent} forecast $\vecg$.
\end{enumerate}
\end{thm}

\noindent Thus, if $\vecf$ and $\vecg$ are coherent  and $\vecf \neq
\vecg$ then neither weakly dominates the other.  The theorem follows from
three propositions of independent interest, stated in the next section.
We close the present section with a corollary.

\begin{cor}\label{cor1}
A forecast $\vecf$ is weakly dominated by a forecast $\vecg \neq \vecf$
if and only if $\vecf$ is
strongly dominated by a coherent forecast.

\end{cor}

\begin{proof}[Proof of Corollary \ref{cor1}]
The right-to-left direction is immediate from Definition \ref{def3w}. For the
left-to-right direction, suppose forecast $\vecf$ is weakly dominated by some
$\vecg\neq \vecf$. Then by Theorem \ref{thm1}(a), $\vecf$ is not coherent. So by
Theorem \ref{thm1}(b), $\vecf$ is strongly dominated by some coherent forecast.
\end{proof}

\section{Three Propositions}\label{threep}

The first proposition is a characterization of coherence. It is due to \cite{def74}.

\begin{defn}\label{nd2}
Let $V=\{(\CF{E_1}(\omega), \cdots,
\CF{E_n}(\omega))\MID\omega\in\Omega\}\subseteq\{0,1\}^n$.
Let the cardinality of $V$ be $k$.
Let $\CONV(V)$ be the \textbf{convex hull} of $V$, i.e., $\CONV(V)$ consists
of all vectors of form $a_1 \vecv_1 + \cdots\ + a_k\vecv_k$, where $\vecv_i\in V$, $a_i \ge 0$,
and $\sum_{i=1}^k a_i = 1$.
\end{defn}

\noindent
The $E_i$ may be related in various ways, so $k < 2^n$ is possible
(indeed, this is the case of interest).

\begin{prop}\label{dlem1}
A forecast $\vecf$ is coherent if and only if $\vecf\in\CONV(V)$.
\end{prop}

The next proposition characterizes scoring rules in
terms of convex functions. Recall that a convex function $\varphi$ on
a convex subset of $\Re^n$ satisfies $\varphi(a \vecx + (1-a) \vecy)
\leq a \varphi(\vecx) + (1-a) \varphi(\vecy)$ for all $0<a<1$ and all $\vecx$, $\vecy$ in
the subset.  {\it Strict} convexity means that the inequality is
strict unless $\vecx=\vecy$. Variants of the following fact are proved
in \cite{Savage71}, \cite{Ban06}, and \cite{Gneiting07}.

\begin{prop}\label{robthm1}
Let $s$ be a proper scoring rule. Then the function $\varphi:[0,1] \to \Re$ defined by $\varphi(x)= -x s(1,x) -(1-x) s(0,x)$ is a bounded, continuous and strictly convex
function, differentiable for $x\in (0,1)$. Moreover,
\begin{equation}\label{reps}
s(i,x) = -\varphi(x) -\varphi'(x)(i-x)\quad \forall x\in(0,1)\,.
\end{equation}
\noindent
Conversely, if a function $s$ satisfies (\ref{reps}), with $\varphi$ bounded, strictly convex and differentiable on $(0,1)$,
and $s$ is continuous on $[0,1]$, then $s$ is a proper scoring rule.
\end{prop}

We note that the right side of (\ref{reps}), which is only defined for $x\in(0,1)$, can be continuously extended to $x=0,1$.
This is the content of the Lemma~\ref{tl} in the next section.
If the extended $s$ satisfies (\ref{reps}) then:
\begin{equation}\label{danAdd1}
  s(0,0)=  - \varphi(0) \quad {\rm and} \quad  s(1,1) = -  \varphi(1)\,.
\end{equation}

Finally, our third proposition concerns a well known property of
Bregman divergences (see, e.g., \citealp{Censor97}).
When we
apply the proposition to the proof of Theorem~\ref{thm1}, $C$ will be
the unit cube in $\Re^n$.

\begin{defn}\label{bregdef}
Let $C$ be a convex subset of $\Re^n$ with non-empty interior.
Let $\Phi:C\rightarrow \Re$ be a strictly convex function, differentiable in the interior of $C$,
whose gradient $\nabla \Phi$ extends to a bounded,  continuous function on $C$.
For $\vecx,\vecy\in C$, the \textbf{Bregman divergence} $d_{\Phi}:C\times C\to \Re$ corresponding
to $\Phi$ is given by
$$
d_{\Phi}(\vecy,\vecx)  = \Phi(\vecy) - \Phi(\vecx) -
\nabla \Phi(\vecx)\cdot(\vecy - \vecx).
$$
\end{defn}

\noindent
Because of the strict convexity of $\Phi$,
$d_{\Phi}(\vecy,\vecx) \geq 0$ with equality if and only if
$\vecy = \vecx$.

\begin{prop}\label{robthm2}
    Let $d_{\Phi}:C\times C\to \Re$ be a Bregman
    divergence, and let $Z \subseteq C$ be a closed convex subset of
    $\Re^n$. For $\vecx \in C \setminus Z$, there exists a unique $\vecpi_{\vecx}
    \in Z$, called the \textbf{projection of $\vecx$ onto $Z$}, such that
$$
d_{\Phi}(\vecpi_{\vecx},\vecx) \leq d_{\Phi}(\vecy,\vecx) \quad \forall
\vecy \in Z\,.
$$
\noindent
Moreover,
\begin{equation}\label{pythagoras}
d_{\Phi}(\vecy,\vecpi_{\vecx}) \leq d_{\Phi}(\vecy,\vecx) -
d_{\Phi}(\vecpi_{\vecx},\vecx) \quad \forall \vecy \in Z,\, \vecx \in C\setminus Z\,.
\end{equation}
\end{prop}

Its worth observing that Proposition \ref{robthm2}
also holds if $\vecx \in Z$, in which case $\vecpi_{\vecx} = \vecx$ and (\ref{pythagoras}) is trivially satisfied.

\section{Proof of Theorem \ref{thm1}}

The main idea of the proof is more apparent when $s$ is bounded. So we consider this
case on its own before allowing $s$ to reach $+\infty$.

\noindent{\bf Bounded Case.}

Suppose $s$ is bounded.
In this case, the derivative of the corresponding $\varphi$ from Eq.~(\ref{reps}) in Proposition~\ref{robthm1}
is continuous and bounded all the way up to the boundary of $[0,1]$.

Let $\vecf \in [0,1]^n$ be a forecast and, for $\omega \in \Omega$,  let $\vecv_\omega\in V$ be the vector with components $C_{E_i}(\omega)$.
Let $\Phi(\vecx) = \sum_{i=1}^n \varphi(x_i)$.
Then
\begin{eqnarray}\nonumber
\PE(\omega,\vecf) & = & \sum_{i = 1}^ns(\CF{E_i}(\omega),f_i)
\quad\text{[Definition \ref{def2}]}\\ \nonumber
               & = & \sum_{i = 1}^n -\varphi(f_i) -
               \varphi'(f_i)(\CF{E_i}(\omega) - f_i)
               \quad\text{[Proposition \ref{robthm1}]}\\ \nonumber
               & = & d_{\Phi}(\vecv_\omega,\vecf) - \sum_{i = 1}^n \varphi(\CF{E_i}(\omega)) \quad\text{[Definition~\ref{bregdef}]}\\
               & = & d_{\Phi}(\vecv_\omega,\vecf) + \sum_{i=1}^n
               s(C_{E_i}(\omega),C_{E_i}(\omega)) \quad\text{[Equation
                 \ref{danAdd1}]}. \label{repr}
\end{eqnarray}

Now assume that $\vecf$ is incoherent which, by Proposition~\ref{dlem1},
means that $\vecf\not\in \CONV(V)$. According to
Eq.~(\ref{pythagoras}) of Proposition~\ref{robthm2}, there exists a
$\vecg\in \CONV(V)$, namely the projection of $\vecf$ onto $\CONV(V)$,
such that $d_{\Phi}(\vecy, \vecg)\leq d_{\Phi}(\vecy, \vecf) -
d_{\Phi}(\vecg,\vecf)$ for all $\vecy \in \CONV(V)$ and hence, in
particular, for $\vecy \in V$.  
Since $d_{\Phi}(\vecg,\vecf)>0$ this proves part $(b)$ of
Theorem~\ref{thm1}.

To prove part $(a)$ first note that weak dominance of $\vecf$ by
$\vecg$ means that $d_{\Phi}(\vecv_\omega,\vecg) \leq
d_{\Phi}(\vecv_\omega,\vecf)$ for all $\vecv_\omega\in V$, by Eq.
(\ref{repr}). In this case, $d_{\Phi}(\vecy,\vecg) \leq
d_{\Phi}(\vecy,\vecf)$ for all $\vecy \in \CONV(V)$, since
$d_{\Phi}(\vecy,\vecg) - d_{\Phi}(\vecy,\vecf)$ depends linearly on
$\vecy$. If $\vecf$ is coherent, $\vecf\in \CONV(V)$ by
Proposition~\ref{dlem1}, and hence $d_{\Phi}(\vecf,\vecg) \leq
d_{\Phi}(\vecf,\vecf)=0$. This implies that $\vecg=\vecf$.

\bigskip
\noindent{\bf Unbounded Case.}

Next, consider the case when $s$ is unbounded. In this case, the
derivative of the corresponding $\varphi$ from Proposition~\ref{robthm1}
diverges either at $0$ or $1$, or at both values, and hence we can not
directly apply Proposition~\ref{robthm2}. Eq.~(\ref{repr}) is still valid,
with both sides of the equation possibly being $+\infty$. However, if
$\vecf$ lies either in the interior of $[0,1]^n$, or on a point on the
boundary where the derivative of $\Phi(\vecx)=\sum_i \varphi(x_i)$ does
not diverge, an examination of the proof of Proposition~\ref{robthm2}
shows that the result still applies, as we show now.

If $\nabla\Phi(\vecf)$ is finite, the minimum of $\Phi(\vecy) -
\nabla\Phi(\vecf)\cdot \vecy$ over $\vecy \in \CONV(V)$ is uniquely
attained at some $\vecg \in \CONV(V)$. Moreover, $\nabla\Phi(\vecg)$
is necessarily finite. Repeating the argument in the proof of
Proposition~\ref{robthm2} shows that $d_\Phi(\vecy,\vecg) \leq
d_\Phi(\vecy,\vecf)- d_\Phi(\vecg,\vecf)$ for any $\vecy \in
\CONV(V)$, which is the desired inequality needed in the proof of
Theorem~\ref{thm1}$(b)$.
We are thus left with the case in which $\vecf$ lies on an $(n-1)$
dimensional face of $[0,1]^n$ where the normal derivative diverges.
Consider first the case $n=1$. Then either $V=\{0,1\}$, in which case
$\vecf$ is coherent, or $V=\{0\}$ or $\{1\}$, in which case it is
clear that the unique coherent vector $\vecg\in V$ strongly dominates
$\vecf$.

We now proceed by induction on the dimension $n$ of the forecast $\vecf$.
In the $(n-1)$ dimensional
hypercube, either $\vecf$ lies inside or on a point of the boundary where
the normal derivative of $\Phi$ is finite, in which case we have just
argued that there exists a $\tilde \vecg$ that is coherent and satisfies
$\PE(\omega,\tilde \vecg) < \PE(\omega,\vecf)$ for all $\omega$ such that
$\vecv_\omega$ lies in the $(n-1)$ dimensional face. In the other case,
the induction hypothesis implies that we can find such a $\tilde \vecg$.
Note that for all the other $\omega$, $\PE(\omega,\tilde \vecg)=\PE(\omega,\vecf)=\infty$.  Now simply pick an $0<\epsilon<1$
and choose $\vecg_{\epsilon}= (1-\epsilon)\tilde \vecg + \epsilon l^{-1}
\sum_{i=1}^l \vecv_i$, where the $\vecv_i$ denote all the $l$ elements of $V$
outside the $(n-1)$-dimensional hypercube.  Then
$\PE(\omega,\vecg_\epsilon)<\infty$ for all $\omega$ and also, using
Lemma~\ref{tl}, $\lim_{\epsilon\to 0} \PE(\omega,\vecg_\epsilon) =
\PE(\omega,\tilde \vecg)$.  Hence we can choose $\epsilon$ small enough to
conclude that $\PE(\omega,\vecg_\epsilon)<\PE(\omega,\vecf)$ for all
$\omega\in \Omega$. This finishes the proof of part $(b)$ in the
general case of unbounded $s$.

To prove part $(a)$ in the general case, we note that if $\vecf = \sum_i
a_i \vecv_i$ for $\vecv_i\in V$ and $a_i>0$, then necessarily
$d_\Phi(\vecv_i,\vecf)<\infty$. That is, any coherent $\vecf$ is a convex
combination of $\vecv_i \in V$ such that $d_\Phi(\vecv_i,\vecf)<\infty$. This
follows from the fact that a component of $\vecf$ can be $0$ only if this
component is $0$ for all the $\vecv_i$'s. The same is true for the value
$1$. But the $d_\Phi(\vecv,\vecf)$ can be infinite only if some component of
$\vecf$ is $0$ and the corresponding one for $\vecv$ is $1$, or vice versa.

Since $d_\Phi(\vecv_i,\vecf)<\infty$ for the $\vecv_i$ in question,
also $d_\Phi(\vecv_i,\vecg)<\infty$ by Eq.~(\ref{repr}) and the
assumption that $\vecf$ is weakly dominated by $\vecg$. Moreover,
$d_\Phi(\vecv_i,\vecg)-d_\Phi(\vecv_i,\vecf)\leq 0$. But $\sum_{i} a_i
( d_\Phi(\vecv_i,\vecg)-d_\Phi(\vecv_i,\vecf) ) =
d_\Phi(\vecf,\vecg)\geq 0$, hence $\vecf=\vecg$.  \hfill\qed

\section{Proofs of Propositions~\ref{dlem1}--\ref{robthm2}}

\begin{proof}[Proof of Proposition \ref{dlem1}]
Recall that $n$ is the dimension of $\BE$, and that $k$ is the number of
elements in $V$.
Let $X$ be the collection of all nonempty sets of form $\bigcap_{i=1}^n
E^*_i$, where
$E^*_i$ is either $E_i$ or its complement.
($X$ corresponds to the minimal non-empty
regions appearing in the Venn diagram of \BE.)
It is easy to see that:
\begin{enumerate}\setlength{\parskip}{.0cm}
\item\label{dan1} $X$ partitions $\Omega$.
\end{enumerate}
It is also clear that there is a one-to-one correspondence
between $X$ and $V$ with the property that $e\in X$ is mapped to
$\vecv\in V$ such that for all $i\le n$, $e\subseteq E_i$ iff $v_i =
1$. (Here, $v_i$ denotes the $i$th component of $\vecv$.) Thus, there
are $k$ elements in $X$. We enumerate them as $e_1, \cdots, e_k$, and
the corresponding $\vecv$ by $\vecv(e_j)$.
Plainly,
for all $i\le n$, $E_i$ is the disjoint union of
$\{e_j\MID j \le k\ \wedge\ v(e_j)_i = 1\}$, and hence:
\begin{enumerate}\setlength{\parskip}{.0cm}\setcounter{enumi}{1}
\item\label{dan5}
For any measure \PROB\ , $\PROB(E_i) = \sum_{j=1}^k \PROB(e_j)v(e_j)_i$ for all $1\leq i\le n$.
\end{enumerate}
\noindent
For the left-to-right direction of the proposition,
suppose that forecast $f$ is coherent via probability measure $\PROB$. Then
$f_i = \PROB(E_i)$ for all $i\le n$ and hence by (\ref{dan5}),
$f_i = \sum_{j=1}^k \PROB(e_j)v(e_j)_i$.
But the $\PROB(e_j)$ are non-negative
and sum to one by (\ref{dan1}), which shows that $\vecf\in\CONV(V)$.

For the converse, suppose that $\vecf\in \CONV(V)$, which means that there are non-negative $a_j$'s, with $\sum_j a_j =1$, such that $\vecf = \sum_{j=1}^k a_j\vecv(e_j)$. Let $\PROB$
be some  probability measure such that $\PROB(e_j) = a_j$  for all $j\le k$.
By (\ref{dan1}) and the assumption about the $a_i$, it is clear that such
a measure \PROB\ exists. For all $i\le n$, $f_i = \sum_{j=1}^k a_jv(e_j)_i =
\sum_{j=1}^k \PROB(e_j)v(e_j)_i = \PROB(E_i)$ by (\ref{dan5}), thereby exhibiting
$\vecf$ as coherent.
\end{proof}

Before giving the proof of Proposition~\ref{robthm1}, we state and prove the following technical Lemma.

\begin{lem}\label{tl}
Let $\varphi:[0,1] \to \Re$ be bounded, convex and differentiable on $(0,1)$. Then the limits $\lim_{p\to 0,1} \varphi(p)$ and $\lim_{p\to 0,1} \varphi'(p)$ exist, the latter possibly being equal to $-\infty$ at $x=0$ or $+\infty$ at $x=1$. Moreover,
\begin{equation}\label{tle}
\lim_{p\to 0}p \varphi'(p) = \lim_{p\to 1} \varphi'(p)(1-p) =0 \,.
\end{equation}
\end{lem}

\begin{proof}[Proof of Lemma~\ref{tl}]
Since $\varphi$ is convex, the limits $\lim_{p\to 0,1} \varphi(p)$ exist, and they are finite since $\varphi$ is bounded. Moreover, $\varphi'$ is a monotone increasing function,
and hence also $\lim_{p\to 0,1} \varphi'(p)$ exists (but possibly equals $-\infty$ at $x=0$ or $+\infty$ at $x=1$).
Finally, Eq. (\ref{tle}) follows again from monotonicity
of $\varphi'$ and boundedness of $\varphi$, using that $0 = \lim_{p\to 0} \int_{0}^p \varphi'(q) dq \leq \lim_{p\to0} p \varphi'(p)$,
and likewise at $p=1$.
\end{proof}

\begin{proof}[Proof of Proposition \ref{robthm1}]
Let $s$ be a proper scoring rule. For $0<p<1$, let
\begin{equation}\label{px}
\varphi(p) = - \min_{x} \left\{ p s(1,x) +(1-p) s(0,x)\right\} \,.
\end{equation}
By Definition~\ref{def1p}(\ref{deflpb}), the minimum in (\ref{px}) is achieved at $x=p$,
hence $\varphi(p) = - p s(1,p) - (1-p) s(0,p)$.

As a minimum over linear functions, $-\varphi$ is concave; hence $\varphi$ is convex.  Clearly,
$\varphi $ is bounded (because $s \geq 0$ implies, from (\ref{px}),  that $\varphi \leq 0$, but a
convex function can become unbounded only by going to $+\infty$).

The fact that the minimum is achieved uniquely (Def. \ref{def1p})
implies that $\varphi$ is  strictly convex for the following reason.
We take $x,\, y \in (0,1) $ and $0<a<1$ and set $z= ax +(1-a)y$. Then
$\varphi(y)  = -y \, s (1, y) - (1-y) \,  s (0, y) >  -y  \, s (1, x) -
(1-y) \,  s(0,z)$
by uniqueness of the minimizer at $y \neq z$.  Similarly, $\varphi(x) =
-x \,  s(1,x) - (1-x) \,  s(0,x) >  -x  \, s(1,z) -(1-x)  \, s(0,z)$. By
adding $a$ times the first inequality to $1-a$ times the second we obtain
$a \varphi (y) + (1-a) \varphi (x) > -z  \, s(1,z)  - (1-z)  \, s(0,z) = \varphi
(z)$, which is precisely the statement of strict convexity.

Let $\psi(p) = s(0,p) - s(1,p)$. If $\varphi$ is differentiable and
$\varphi'(p) = \psi(p)$ for all $0<p<1$, then (\ref{reps}) is
satisfied, as simple algebra shows.

We shall now show that $\varphi$ is, in fact, differentiable and $\varphi'= \psi$.
For any $p\in (0,1)$ and small enough $\epsilon$, we have
\begin{multline*}
\frac 1{\epsilon}\left(\varphi(p+\epsilon)-\varphi(p)\right) = \psi(p) \\ -
\frac1\epsilon\left[ (p+\epsilon)\left(s(1,p+\epsilon) - s(1,p)\right)
    + (1-p-\epsilon)\left(s(0,p+\epsilon)-s(0,p)\right)\right]\,.
\end{multline*}
Since $(p+\epsilon)s(1,x)+(1-p-\epsilon)s(0,x)$ is minimized at
$x=p+\epsilon$ by Definition~\ref{def1p}(\ref{deflpb}), the last term in square brackets is negative. Hence
$$
\lim_{\epsilon\to 0} \frac 1{\epsilon}\left(\varphi(p+\epsilon)-\varphi(p)\right) \geq \psi(p) \,,
$$
and similarly one shows
$$
\lim_{\epsilon\to 0} \frac 1{\epsilon}\left(\varphi(p)-\varphi(p-\epsilon)\right) \leq \psi(p) \,.
$$
Since $\psi$ is continuous by Definition~\ref{def1p}(\ref{deflpa}),
this shows that $\varphi$ is differentiable, and hence $\psi =
\varphi'$. This proves Eq.~(\ref{reps}). Continuity of $\varphi$ up to
the boundary of $[0,1]$ follows from continuity of $s$ and Lemma~\ref{tl}.

To prove the converse, first note that if $\varphi$ is bounded and convex on $(0,1)$, it can be extended to a continuous function on $[0,1]$, as shown in  Lemma~\ref{tl}. Because of strict convexity of $\varphi$ we have, for $p\in[0,1]$ and $0<x<1$,
\begin{equation}\label{sc}
p s(1,x) + (1-p) s(0,x) = -\varphi(x) - \varphi'(x) (p-x) \geq -\varphi(p) \,,
\end{equation}
with equality if and only if $x=p$.

It remains to show that the same is true for $x\in \{0,1\}$. Consider first the case $x=0$. We have to show that $p s(1,0) + (1-p) s(0,0)> - \varphi(p)$  for $p>0$. By continuity of $s$, Eq. (\ref{reps}) and Lemma~\ref{tl}, we have $s(1,0) = -\varphi(0) - \lim_{p\to 0} \varphi'(p)$, while $s(0,0) = -\varphi(0)$.
If $\lim_{p\to 0} \varphi'(p) = -\infty$, the result is immediate. If $\varphi'(0):=\lim_{p\to 1} \varphi'(p)$ is finite, we have $-\varphi(0) - p \varphi'(0) >  -\varphi(p)$ again by strict convexity of $\varphi$.

Likewise, one shows that  $p s(1,1) + (1-p) s(0,1) > -\varphi(p)$ for $p<1$. This finishes the proof that $s$ is a proper scoring rule.
\end{proof}

\begin{proof}[Proof of Proposition~\ref{robthm2}]
    For fixed $\vecx \in C$, the function $\vecy \mapsto
    d_{\Phi}(\vecy,\vecx)$ is strictly convex, and hence achieves
    a unique minimum at a point $\vecpi_{\vecx}$ in the convex, closed set $Z$.

Let $\vecy\in Z$. For $0\leq \epsilon\leq 1$,
$(1-\epsilon)\vecpi_\vecx + \epsilon \vecy \in Z$, and hence
$d_\Phi( (1-\epsilon)\vecpi_\vecx + \epsilon \vecy,\vecx) -
d_\Phi(\vecpi_\vecx,\vecx)\geq 0$ by the definition of $\vecpi_\vecx$. Since
$d_\Phi$ is differentiable in the first argument, we can divide by
$\epsilon$ and let $\epsilon\to 0$ to obtain
$$
0\leq \lim_{\epsilon\to 0} \frac 1\epsilon \left( d_\Phi( (1-\epsilon)\vecpi_\vecx + \epsilon \vecy,\vecx) -
d_\Phi(\vecpi_\vecx,\vecx)\right)=
\left(\nabla\Phi(\vecpi_{\vecx})- \nabla\Phi(\vecx)\right)\cdot (\vecy -\vecpi_{\vecx}) \,.
$$
The fact that
$$
d_{\Phi}(\vecy,\vecx) -
d_{\Phi}(\vecpi_{\vecx},\vecx) - d_{\Phi}(\vecy,\vecpi_{\vecx}) =
\left(\nabla \Phi(\vecpi_{\vecx})- \nabla \Phi(\vecx)\right)\cdot (\vecy -\vecpi_{\vecx})
$$
proves the claim.
\end{proof}

\section{Generalizations}\label{discussion}

\subsection{Penalty functions}

Theorem~\ref{thm1} holds for a larger class of penalty functions.
In fact, one can use different proper scoring rules for every event,
and replace (\ref{penalty}) by
\begin{equation} \nonumber
\PE(\omega,\vecf) = \sum_{i \le n}s_i(\CF{E_i}(\omega),f_i)\,,
\end{equation}
where the $s_i$ are possibly distinct proper scoring rules. In this way,
forecasts for some events can be penalized differently than others.
The relevant Bregman divergence in this case is given by $\Phi(\vecx)=\sum_i \varphi_i(x_i)$,
where $\varphi_i$ is determined by $s_i$ via (\ref{reps}).
Proof of this generalization closely follows the argument given above, so it is omitted.
Additionally, by considering more general convex functions $\Phi$ our argument generalizes
to certain non-additive penalties.

\subsection{Generalized scoring rules}

\subsubsection{Non-uniqueness}
If one relaxes the condition of {\it unique} minimization in Definition~\ref{def1p}(\ref{deflpb}),
a weaker form of Theorem~\ref{thm1} still holds. Namely, for any incoherent forecast $\vecf$ there
exists a coherent forecast $\vecg$ that weakly dominates $\vecf$. Strong dominance will not hold in
general, as the example of $s(i,x)\equiv 0$ shows.

Proposition~\ref{robthm1} also holds in this generalized case, but the function $\varphi$ need not
be {\it strictly} convex. Likewise, Proposition~\ref{robthm2} can be generalized to merely convex
(not necessarily strictly convex) $\Phi$ but in this case the projection $\vecpi_\vecx$ need not be unique.
Eq.~(\ref{pythagoras}) remains valid.

\subsubsection{Discontinuity}
A generalization that is more interesting mathematically is to discontinuous scoring rules.
Proposition~\ref{robthm1} can be generalized to scoring rules that satisfy neither the continuity condition in
Definition~\ref{def1p} nor unique minimization.
(This is also shown in \citealp{Gneiting07}).

\begin{prop}\label{Robthm1}
Let $s:\{0,1\}\times [0,1]\to [0,\infty]$ satisfy
\begin{equation}\label{newc}
p s(1,x) + (1-p) s(0,x) \geq p s(1,p) + (1-p) s(0,p) \quad \forall x, p\in[0,1]\,.
\end{equation}
Then the function $\varphi:[0,1] \mapsto \Re$ defined by $\varphi(x) = - x s(1,x) - (1-x) s(0,x)$ is bounded and convex. Moreover, there exists a
 monotone non-decreasing function
$\psi:[0,1]\mapsto\Re\cup\{\pm\infty\}$, with the property that
\begin{align}\label{proppsi}
\psi(x) &\geq \lim_{\epsilon \to 0} \frac 1\epsilon \left(
   \varphi(x)-\varphi(x-\epsilon)\right) \quad \forall x\in(0,1]\,, \\
\psi(x) &\leq \lim_{\epsilon \to
   0} \frac 1\epsilon \left( \varphi(x+\epsilon)-\varphi(x)\right) \quad
\forall x\in [0,1)\,, \label{proppsi2}
\end{align}
such that
\begin{equation}\label{rreps}
s(i,x) = -\varphi(x) -\psi(x)(i-x)\quad \forall x\in(0,1)\,.
\end{equation}
Function $\varphi$ is strictly convex if and only if the inequality (\ref{newc}) is strict for $x\neq p$.

Conversely, if $s$ is of the form (\ref{rreps}), with $\varphi$ bounded
and convex and $\psi$ satisfying
(\ref{proppsi})--(\ref{proppsi2}), then $s$ satisfies (\ref{newc}).
\end{prop}

It is a fact \citep{Polya} that every convex function $\varphi$ on $[0,1]$
is continuous on $(0,1)$ and has a right and left derivative, $\psi_R$ and $\psi_L$ (defined
by the right sides of (\ref{proppsi2}) and (\ref{proppsi}), respectively) at every point (except
the endpoints, where it has only a right or left derivative, respectively). Both $\psi_R$ and $\psi_L$
are non-decreasing functions, and $\psi_L(x) \leq \psi_R(x)$ for all $x\in (0,1)$. Except
for countably many points, $\psi_L(x) = \psi_R(x)$, i.e., $\varphi$ is differentiable.
Eqs.~(\ref{proppsi})--(\ref{proppsi2}) say that $\psi_L(x)\leq \psi(x)\leq \psi_R(x)$.

Note that although $s(0,x)$ and $s(1,x)$ may be discontinuous, the combination
$\varphi(x)= -x s(1,x)-(1-x) s(0,x)$ is continuous. Hence, if $s(0,x)$ jumps
up at a point $x$, $s(1,x)$ has to jump down by an amount proportional to $(1-x)/x$.

The proof of Proposition~\ref{Robthm1} is virtually the same as
the proof of Proposition~\ref{robthm1}, so we omit it.

\subsection{Open question}
Whether Theorem~\ref{thm1} holds for this generalized notion of a discontinuous scoring rule
remains open. The proof of Theorem~\ref{thm1} given here
does not extend to the discontinuous case, since for inequality
(\ref{pythagoras}) to hold, differentiability of $\Phi$ is necessary, in general.

\bibliographystyle{plainnat}
\bibliography{EllRobDan}

\vspace*{.5in}

\begin{center}
\begin{tabular}{lll}
\begin{tabular}{l}
Joel Predd\\Rand Corporation\\
4570 Fifth Avenue, Suite 600\\
Pittsburgh, PA 15213\\
jpredd@rand.org\\\end{tabular}&
\begin{tabular}{l}
Robert Seiringer\\Dept.\ of Physics\\
Princeton University\\Princeton NJ 08540\\
rseiring@princeton.edu\\\end{tabular}\\[2cm]
\begin{tabular}{l}
Elliott Lieb\\Depts.\ of Mathematics and Physics\\
Princeton University\\Princeton NJ 08540\\
lieb@princeton.edu\\\end{tabular}&
\begin{tabular}{l}
Daniel Osherson\\Dept.\ of Psychology\\
Princeton University\\Princeton NJ 08540\\
osherson@princeton.edu\\\end{tabular}\\[2cm]
\begin{tabular}{l}
Vincent Poor\\Dept.\ of Electrical Engineering\\
Princeton University\\Princeton NJ 08540\\
poor@princeton.edu\\\end{tabular}&
\begin{tabular}{l}
Sanjeev Kulkarni\\Dept.\ of Electrical Engineering\\
Princeton University\\Princeton NJ 08540\\
kulkarni@princeton.edu\\\end{tabular}\\
\end{tabular}
\end{center}

\end{document}